\documentclass[12pt,draftcls,onecolumn]{IEEEtran}

\IEEEoverridecommandlockouts                              

\usepackage{cite}
\usepackage{enumerate}
\usepackage[english]{babel}
\usepackage[utf8]{inputenc}
\usepackage{algorithm}
\usepackage[noend]{algpseudocode}
\usepackage{amsmath}
\usepackage{amssymb}
\usepackage{amsthm}
\usepackage{graphicx}
\usepackage{hyperref}
\usepackage{graphics}
\newtheorem{lemma}{Lemma}
\newtheorem{theorem}{Theorem}
\newtheorem{corollary}{Corollary}

\newtheorem{remark}{Remark}
\newtheorem{assumption}{Assumption}
\newcommand{\R}{\mathbb{R}}
\newcommand{\E}{\mathbb{E}}
\newcommand{\F}{\mathcal{F}}
\newcommand{\g}{\gamma}
\newcommand{\V}{\text{val}}

\usepackage[dvipsnames]{xcolor}
\usepackage{caption}
\usepackage{subcaption}

\usepackage{tikz}
\usepackage{textcomp}
\usepackage{hyperref}
\usepackage{lipsum}

\newcommand\copyrighttext{%
  \footnotesize 0018-9286 \textcopyright 2021 IEEE. Personal use is permitted, but republication/redistribution requires IEEE permission. A version of this paper is accepted for publication at the IEEE Transactions on Automatic Control.
  DOI: 10.1109/TAC.2022.3159453}
\newcommand\copyrightnotice{%
\begin{tikzpicture}[remember picture,overlay]
\node[anchor=south,yshift=10pt] at (current page.south) {\fbox{\parbox{\dimexpr\textwidth-\fboxsep-\fboxrule\relax}{\copyrighttext}}};
\end{tikzpicture}%
}

\title{
A Generalized Minimax Q-learning Algorithm for Two-Player Zero-Sum Stochastic Games
}

\author{Raghuram Bharadwaj Diddigi, Chandramouli Kamanchi, Shalabh Bhatnagar
\thanks{The authors are with the Department of Computer
Science and Automation, Indian Institute of Science, Bangalore, 560012,
India (e-mails: raghub@iisc.ac.in; chandramouli@iisc.ac.in; shalabh@iisc.ac.in).}%
}

\begin{document}

\maketitle
\thispagestyle{empty}
\pagestyle{empty}
\copyrightnotice
\begin{abstract}
We consider the problem of two-player zero-sum games. This problem is formulated as a min-max Markov game in the literature. The solution of this game, which is the min-max payoff, starting from a given state is called the min-max value of the state. In this work, we compute the solution of the two-player zero-sum game utilizing the technique of successive relaxation that has been successfully applied in the literature to compute a faster value iteration algorithm in the context of Markov Decision Processes. We extend the concept of successive relaxation to the setting of two-player zero-sum games. We show that, under a special structure on the game, this technique facilitates faster computation of the min-max value of the states. We then derive a generalized minimax Q-learning algorithm that computes the optimal policy when the model information is not known. Finally, we prove the convergence of the proposed generalized minimax Q-learning algorithm utilizing stochastic approximation techniques, \textcolor{black}{under an assumption on the boundedness of iterates}. Through experiments, we demonstrate the effectiveness of our proposed algorithm. 
\end{abstract}

\section{Introduction}
In two-player zero-sum games, there are two agents that are competing against each other in a common environment. Based on the actions taken by the agents, they receive a payoff corresponding to the current state and the environment transitions to the next state. The objective of an agent (say agent 1) is to compute a sequence of actions starting from a given state to maximize the total discounted payoff. On the other hand, the objective of the second agent (agent 2) is to compute a sequence of actions that minimizes the total discounted payoff. This problem is formulated as a Markov game and the value that is obtained as the min-max of the total expected discounted payoff starting from state $i$ is called the min-max value of state $i$. The policies that achieve this min-max value are the optimal policies of the agents. 

When the model information of the environment is known, a Bellman operator for the two-player zero-sum game \cite{bertsekas1995dynamic} is constructed and a fixed point iteration scheme analogous to the value iteration is used to compute the min-max value. However, in most two-player zero-sum game settings, the model information is assumed unknown to the players and the objective is to compute optimal policies utilizing the state and payoff samples obtained from the environment. 

In our work, we construct a modified min-max Q-Bellman operator by using the technique of successive relaxation for the Markov games and prove that the contraction factor is at most the contraction factor of the standard min-max Bellman operator. This implies that, when the model information is known, the min-max value can be computed faster using our proposed scheme. We then proceed to develop a generalized minimax Q-learning algorithm based on the modified min-max Q-Bellman operator. 

The minimax Q-learning algorithm has been presented in \cite{littman1994markov}. 
Two-player general sum games are those where the payoffs of the agents are unrelated in general. If the payoff of an agent is the negative of the payoff of another agent, the game reduces to a zero-sum game. A Nash Q-learning algorithm for solving general sum games is proposed in \cite{hu2003nash}. In \cite{littman2001friend}, Friend-or-Foe (FF) Q-learning for general sum games is proposed and is shown to have stronger convergence properties compared to Nash Q-learning. A generalization of Nash Q-learning and FF Q-learning, namely correlated Q-learning, is discussed in \cite{greenwald2003correlated}. In \cite{bowling2001rational}, desirable properties for an agent learning in multi-agent scenarios are studied and a new learning algorithm namely ``WoLF policy hill climbing" is proposed. \textcolor{black}{Surveys of algorithms for multi-agent learning and multi-agent Reinforcement learning are provided in \cite{bu2008comprehensive,zhang2019multi}}. 

We now discuss some variants of minimax Q-learning in the literature. In \cite{dahl2000minimax}, the minimax TD-learning algorithm that utilizes the concept of temporal difference learning is proposed. The minimax version of the Deep Deterministic policy gradient algorithm has been recently developed in \cite{li2019robust}. However, no convergence proofs or theoretical guarantees are provided.

The concept of successive relaxation in the context of Markov Decision Processes (MDPs) has been first applied in \cite{reetz1973solution}. In our recent work \cite{8731666}, we have proposed successive over-relaxation Q-learning for model-free MDPs (i.e., in the single-agent scenario). In this work, we extend the concept of successive relaxation to the two-player zero-sum games and propose a provably convergent generalized minimax Q-learning. The contributions of the paper are as follows:
\begin{itemize}
    \item We present a modified min-max Q-Bellman operator for two-player zero-sum Markov games and show that the operator is a max-norm contraction.
    \item We show that under some assumptions, the contraction factor of the modified min-max Q-Bellman operator is smaller than the standard min-max Q-Bellman operator.
    \item \textcolor{black}{We propose a model-free generalized minimax Q-learning algorithm and prove its almost sure convergence using ODE based analysis of stochastic approximation, \textcolor{black}{under an assumption on the boundedness of iterates}}
    \item  \textcolor{black}{We discuss an interesting relation between standard minimax Q-learning and our proposed algorithm.}
    \item Finally, through experimental evaluation, we show that our proposed algorithm has a better performance compared to the standard minimax Q-learning algorithm. 
\end{itemize}
\textcolor{black}{We note here that the Successive Over Relaxation (SOR) technique utilized to derive our algorithm and stochastic approximation arguments employed in the convergence analysis are well-known in the literature. Our contribution comprises of applying these techniques to derive and analyze a generalized minimax Q-Learning algorithm that has faster convergence.}

\section{Background and Preliminaries}
In this paper, we consider the setting of two-player zero-sum Markov games \cite{filar2012competitive}. The two players in the game are referred to as agent 1 and agent 2. A two-player zero-sum Markov game is characterized by the tuple $(S,U,V, p,r,\alpha)$ where $S$ is the set of states that both the agents observe, $U$ is the finite set of actions of agent 1, $V$ is the finite set of actions of agent 2, $p$ denotes the transition probability rule, i.e., $p(j|i,u,v)$ denotes the probability of transition to state $j \in S$ from state $i \in S$ when actions $u \in U$ and $v \in V$ are chosen by the agents 1 and 2, respectively. Let $r(i,u,v)$ denote the single-stage payoff obtained by the agent 1 in state $i$ when actions $u$ and $v$ are chosen by agents 1 and 2, respectively. Note that, in the case of a zero-sum Markov game, the payoff of the agent 2 is the negative of the payoff obtained by the agent 1. Also, $0 \leq \alpha < 1$ denotes the discount factor. The goals of the two agents in the Markov game are to individually learn the optimal policies $\pi_1: S \xrightarrow{} \Delta^{|U|} $ and $\pi_2:S \xrightarrow{} \Delta^{|V|}$, respectively, where $\Delta^{d}$ denotes the probability simplex in $\R^d$ and $\pi_1(i)$ (resp. $\pi_2(i)$) indicates the probability distribution over actions to be taken by the agent 1 (resp. agent 2) in state $i$ that maximizes (resp. minimizes) the discounted objective given by:
\begin{align}\label{two-cost}
    \displaystyle \min_{\pi_2} \max_{\pi_1} \sum_{t = 0}^{\infty} \E \Big[\sum_{u = 1}^{|U|}\sum_{v=1}^{|V|} \alpha^{t} x^{u}_{t}y^{v}_{t} r(s_t,u,v) \mid s_{0} = i \Big],
\end{align}
where $s_{t}$ is the state of the game at time $t$, $\pi_{1}(s_t)=(x^{u}_{t})^{|U|}_{u=1} \in \Delta^{|U|}, \pi_{2}(s_t)=(y^{v}_{t})^{|V|}_{v=1}  \in \Delta^{|V|}, ~ \forall t \geq 0 $ and $\E[.]$ is the expectation taken over the states obtained over time $t = 1,\ldots,\infty$. Let \textcolor{black}{$J^*(i)$} denote the min-max value in state $i$ obtained by solving \eqref{two-cost}. It can be shown (\textcolor{black}{ \cite[Chapter~7]{bertsekas1996neuro}}) that the min-max value function, \textcolor{black}{$J^*$}, satisfies \textcolor{black}{the following fixed point equation in $J \in \R^{|S|}$,} given by: 
\begin{align}\label{v-eq}
    J(i)= \V[Q(i)], ~ \forall i\in S, 
\end{align}
where $Q(i)$ is a matrix of size $|U| \times |V|$, whose $(u,v)^{th}$ entry is given by $Q(i,u,v) = r(i,u,v) + \alpha \displaystyle\sum_{j \in S} p(j|i,u,v) J(j)$ and the function $\V[A]$, for a given $m \times n$ matrix $A$, is defined as follows:
\begin{align}\label{v-def}
    \V[A] = \displaystyle \min_{y} \max_{x} x^TAy,
\end{align}
where $x \in \Delta^{m}$ and $y \in \Delta^{n}$, respectively. 
The system of equations in \eqref{v-eq} can be rewritten as:
\begin{align}\label{fp-vi}
    J = TJ,
\end{align}
where the operator $T$, \textcolor{black}{for a given $J \in \R^{|S|}$,} is defined as: 
\begin{align}
    (TJ)(i) = \V[Q(i)], ~ \forall i \in S.
\end{align}
The operator $T$ and the set of equations \eqref{v-eq} are analogous to the Bellman operator and the Bellman optimality condition, respectively, for Markov Decision Processes (MDPs) \cite{bertsekas1996neuro}.

\section{The Proposed Algorithm}\label{PA}
We describe a single iteration of the synchronous version \cite{borkar2009stochastic} of our proposed algorithm in Algorithm \ref{alg:minimax Q-learning} below. At each iteration $n$, Q-values of all the state-action tuple $Q(i,u,v)$ are updated as shown in the step 4 of Algorithm \ref{alg:minimax Q-learning}.


\begin{algorithm}[ht]
\caption{Generalized minimax Q-Learning}\label{alg:minimax Q-learning}
\textcolor{black}{
\hspace*{\algorithmicindent}\textbf{Input:}\\ 
\hspace*{\algorithmicindent}\textbf{$w$}: Choose $w \in (0,w^*]$ (with $w^*$ as in \eqref{w-star}), as the relaxation parameter. \\
\hspace*{\algorithmicindent}\textbf{$\{Y_n\}$}: a sequence of vectors of size $|S\times U \times V|$, where $Y_n(i,u,v)$ indicates the next state obtained when actions $u,v$ are taken in state $i$. \\
\hspace*{\algorithmicindent}\textbf{$r(i,u,v)$}: single-stage payoff available when actions $u,v$ are chosen in state $i$. \\
\hspace*{\algorithmicindent}\textbf{$\g_{n}$}: step-size chosen at time $n$.\\
\hspace*{\algorithmicindent}\textbf{$Q_{n}$}: the estimate of $Q^{\dagger}$ (see \eqref{MQU}) at time $n$.
\begin{algorithmic}[1]
\Procedure{Generalized minimax Q-Learning:}{}
\For{$(i,u,v) \in S \times U \times V $}
\State $d_{n+1}(i,u,v) = w\Big(r(i,u,v)+\alpha\V [Q_n(Y_{n}(i,u,v))]\Big)+(1-w)\V[Q_n(i)]$
\State $Q_{n+1}(i,u,v)= \big{(}1 - \g_{n}\big{)}Q_{n}(i,u,v) + \g_{n}d_{n+1}(i,u,v)$
\EndFor
\State \textbf{return} $Q_{n+1}$ 
\EndProcedure
\end{algorithmic}
}
\end{algorithm}

\begin{remark}
Note that the step 3 of Algorithm \ref{alg:minimax Q-learning} requires computation of $\V[Q_{n}(.)]$ which is a linear program. Also observe that the generalized minimax Q-learning algorithm only requires an additional computation of $\V[Q_n(i_n)]$ compared to the standard minimax Q-learning.
\end{remark}
\begin{remark}
Let $Q_{N}$ be the solution obtained by Algorithm \ref{alg:minimax Q-learning} upon termination after $N$ iterations. Then the approximate min-max value, $\tilde{J}(i)$ for a given state $i$ is obtained as follows:
\begin{align*}
    \tilde{J}(i) = \V[Q_{N}(i)]
\end{align*}
and the corresponding approximate policies of the agents are obtained as:
\begin{align*}
    (\tilde{\pi}_1(i),\tilde{\pi}_2(i)) \in \arg \V[Q_{N}(i)].
\end{align*}
\end{remark}

\section{Convergence Analysis}
Let $\Delta^{d}:=\{x \in \R^{d}: x_{i}\geq 0, \sum^{d}_{i=1}x_{i}=1\}$ denote the probability simplex in $\R^{d}$. For matrix $A \in \R^{m \times n}, x \in \Delta^{m}$ and $y \in \Delta^{n}$, recall that the value of the matrix $A$ is defined as $\V[A]:= \displaystyle\min_{y \in \Delta^{n}}\displaystyle\max_{x \in \Delta^{m}} x^TAy$. Note that the norm considered in this section is the max-norm, i.e., norm of the vector $x \in \R^{d}$ is $\|x\|:=\displaystyle\max_{1\leq i\leq d}|x(i)|$. We first derive a few properties of the $\V[.]$ operator that would be used in the subsequent analysis.
\begin{lemma}
\label{l1}
Suppose $B=[b_{ij}],C=[c_{ij}] \in \R^{m \times n}$, then
\begin{align*}
|\V[B]-\V[C]|\leq \displaystyle\max_{i,j}|b_{ij}-c_{ij}|=\|B-C\|
\end{align*}
\end{lemma}
\begin{proof}
\begin{align*}
\V[B]-\V[C]=& \min_{y}\max_{x} x^TBy- \min_{y}\max_{x}x^TCy \\
\leq &-\min_{y}\bigg{\{}\max_{x}x^TCy-\max_{x} x^TBy\bigg{\}}\\
= & \max_{y}\bigg{\{}\max_{x} x^TBy-\max_{x}x^TCy\bigg{\}} \\
\leq & \max_{y}\max_{x} x^T(B-C)y \\
\leq & \max_{y}\max_{x}\bigg{|}\sum_{1\leq i,j \leq n} x_{i}(b_{ij}-c_{ij})y_{j}\bigg{|} \\
  \leq & \displaystyle\max_{i,j}|b_{ij}-c_{ij}|\max_{y}\max_{x}\bigg{|}\sum_{1\leq i,j \leq n} x_{i}y_{j}\bigg{|}\\
  =&\displaystyle\max_{i,j}|b_{ij}-c_{ij}|.
\end{align*}
In the above, $x \in \Delta^m$, and $y \in \Delta^n$, respectively. Similarly,
\begin{align*}
\V[B]-\V[C] = & \min_{y}\max_{x} x^TBy- \min_{y}\max_{x}x^TCy \\
\geq & \min_{y}\bigg{\{}\max_{x}x^TBy-\max_{x} x^TCy\bigg{\}}\\
\geq & \min_{y}\bigg{\{}-\max_{x} x^T(C-B)y\bigg{\}} 
\end{align*}
\begin{align*}
= & -\max_{y}\max_{x} x^T(C-B)y \\
\geq & -\max_{y}\max_{x}\bigg{|}\sum_{1\leq i,j \leq n} x_{i}(b_{ij}-c_{ij})y_{j}\bigg{|} \\
\geq & -\displaystyle\max_{i,j}|b_{ij}-c_{ij}|\max_{y}\max_{x}\bigg{|}\sum_{1\leq i,j \leq n} x_{i}y_{j}\bigg{|}\\
=&-\displaystyle\max_{i,j}|b_{ij}-c_{ij}|.
\end{align*}
Therefore $|\V[B]-\V[C]|\leq \max_{i,j}|b_{ij}-c_{ij}|=\|B-C\|.$
Note the repeated application of the facts 
$\displaystyle\sum_{i,j} x_{i}y_{j}=1$, $\sup{(a_{n}+b_{n})}\leq \sup{a_n}+\sup{b_n}$ in the arguments. This completes the proof.
\end{proof}
\begin{corollary}\label{l2}
Consider $B = [b_{ij}] \in \R^{m \times n}$, then 
\begin{align*}
|\V[B]|\leq \displaystyle \|B\|.
\end{align*}
\end{corollary}
\begin{proof}
Using Lemma \ref{l1} with $C = 0$, we get:
\begin{align*}
    |\V[B]| &\leq \displaystyle\max_{i,j}|b_{ij}|
    =\displaystyle \|B\|. \qedhere
\end{align*}
\end{proof}

\begin{lemma}
\label{l3}
Let $E=[e_{ij}]_{m \times n}$, where $e_{ij}=1 ~ \forall i,j.$ Then for constants $\beta$,$k$ $\in$ $\R$ and $A \in \R^{m \times n},$
\begin{align*}
    \V[\beta A+kE]=\beta \V[A]+k. 
\end{align*}
\end{lemma}
\begin{proof} By definition of the $\V$ operator, for $x \in \Delta^m$, $y \in \Delta^n$,
\begin{align*}
\V[\beta A+kE] &=\displaystyle\min_{y}\max_{x} x^T(\beta A+kE)y \\ 
            &=\beta \displaystyle\min_{y}\max_{x} x^TAy+k ~ ( \text{since } x^TEy=1)\\
            &=\beta \V[A]+k. \qedhere
\end{align*}
\end{proof}
Recall that for a given stochastic game $(S,U,V,p,r,\alpha)$ the min-max value function $J^*$ satisfies \cite{shapley1953stochastic} the system of equations, 
\begin{align*}
    J(i)=\V[Q(i)], ~ \forall i \in S,
\end{align*}
where $Q(i)$ is a $|U| \times |V|$ matrix
with $(u,v)^\text{th}$ entry 
\begin{align*}
 Q(i,u,v)= r(i,u,v)+\alpha\displaystyle\sum_{j \in S}p(j|i,u,v)J(j),
\end{align*}
and the system of equations can be reformulated as the fixed point equation, $TJ=J$, with $T$ being a contraction  under the max-norm with contraction factor $\alpha$.\\
We define a quantity $w^*$ as follows:
\begin{align}
    w^* \triangleq \displaystyle\min_{i,u,v}\Bigg{\{}\frac{1}{1-\alpha p(i|i,u,v)}\Bigg{\}}.
    \label{w-star}
\end{align}
As the probabilities $p(i|i,u,v) \geq 0, ~ \forall (i,u,v) $, it is clear that $w^* \geq 1$. For $0 < w \leq w^*$, we now define a modified operator 
$T_{w}:\R^{|S|} \rightarrow \R^{|S|}$  as follows \cite{reetz1973solution}:
$$(T_{w}J)(i)=w ~ (TJ)(i)+(1-w)J(i),$$ where $w$ represents a prescribed relaxation factor. Note that $T_{w}$ is in general not a convex combination of $T$ and the identity operator $I$ since we allow $w \geq 1$ as $w^* \geq 1$ (see above).
Let $J^*$ denote the min-max value of the Markov game. Therefore, $TJ^* = J^*$.
Now, 
\begin{align} \label{j-star-same}
    (T_wJ^*)(i) &= w(TJ^*)(i) + (1-w)J^*(i) \nonumber \\
                &= wJ^*(i) + (1-w)J^*(i) \nonumber \\ 
                &= J^*(i).
\end{align}
Therefore, the min-max value function $J^*$ is also a fixed point of $T_{w}.$

Next, we derive a modified min-max Q-Bellman operator for the two-player zero-sum game. 
Let $Q^{\dagger}(i,u,v)$ be defined as follows:
\begin{align}
Q^{\dagger}(i,u,v):=& w\bigg{(}r(i,u,v)+\alpha \sum^{M}_{j=1}p(j|i,u,v)J^*(j)\bigg{)} \nonumber \\ & +(1-w)J^*(i).\label{MQU}
\end{align}
Now let \begin{align*} Q^{*}(i,u,v)=\bigg{(}r(i,u,v)+\alpha \sum^{M}_{j=1}p(j|i,u,v)J^*(j)\bigg{)}\end{align*}
Let $E=[e_{ij}]_{m \times n}$ with $e_{ij}=1, ~ \forall i,j$. Then,
\begin{align*}
    \V[Q^{\dagger}(i)]&=\V[wQ^{*}(i)+(1-w)J^{*}(i)E] \\
    &=w\V[Q^{*}(i)]+(1-w)J^{*}(i) \text{ (from Lemma \ref{l3})}\\
    &=w(TJ^{*})(i)+(1-w)J^{*}(i)\\
    &=(T_wJ^*)(i)=J^*(i) \text{ (from \eqref{j-star-same})}.
\end{align*}
Hence the equation \eqref{MQU} can be rewritten as follows:
\begin{align}
\label{Q-star}
Q^{\dagger}(i,u,v)=& w\bigg{(}r(i,u,v)+\alpha \sum^{M}_{j=1}p(j|i,u,v)\V[Q^{\dagger}(j)]\bigg{)} \nonumber \\ & +(1-w) \V[Q^{\dagger}(i)].
\end{align}
Let $H_w:\R^{|S\times U \times V|} \rightarrow \R^{|S \times U \times V|}$ be defined as follows. For $Q \in \R^{|S\times U \times V|} $,
\begin{align*}
(H_{w}Q)(i,u,v):=& w\bigg{(}r(i,u,v)+ \alpha \sum^{M}_{j=1}p(j|i,u,v) \V[Q(j)]\bigg{)} \nonumber \\ & +(1-w) \V[Q(i)].
\end{align*}
$H_w$ is the modified Q-Bellman operator for the two-player zero-sum Markov game.  
\begin{lemma}
For $0<w \leq w^*$ with $w^*$ as in \eqref{w-star}, the map $H_w:\R^{|S\times U \times V|} \rightarrow \R^{|S \times U \times V|}$ is a max-norm contraction and $Q^{\dagger}$ is the unique fixed point of $H_w$.
\label{l5}
\end{lemma}
\begin{proof} From equation \eqref{Q-star}, $Q^{\dagger}$ is a fixed point of $H_{w}$. Therefore, it is enough to show that $H_{w}$ is a contraction operator (which will also ensure its uniqueness).
For $P,Q \in \R^{|S\times U \times V|}$, we have
\begin{align}
 &\bigg{|}(H_{w}P-H_{w}Q)(i,u,v)\bigg{|} \nonumber \\
=&\bigg{|}w \alpha \displaystyle\sum^{M}_{j=1}p(j|i,u,v)\big{(}\V[P(j)]-\V[Q(j)]\big{)} \nonumber \\
& \hspace{2.5cm} +(1-w)\big{(}\V[P(i)]-\V[Q(i)]\big{)}\bigg{|} \nonumber \\
=&\bigg{|}w \alpha \displaystyle\sum^{M}_{j=1, j\neq i}p(j|i,u,v)\big{(}\V[P(j)]-\V[Q(j)]\big{)} \nonumber\\
& \hspace{0.1cm} + (1-w+w\alpha p(i|i,u,v))\big{(}\V[P(i)]-\V[Q(i)]\big{)}\bigg{|} \nonumber  \\
\leq &\bigg{|}w \alpha \displaystyle\sum^{M}_{j=1, j\neq i}p(j|i,u,v)\big{(}\V[P(j)]-\V[Q(j)]\big{)}\bigg{|}  \nonumber \\
&\hspace{0.1cm}+\big{|}(1-w+w\alpha p(i|i,u,v))\big{|}\bigg{|}\big{(}\V[P(i)]-\V[Q(i)]\big{)}\bigg{|} \label{w-con}
\end{align}
\begin{align}
\leq & w \alpha \displaystyle\sum^{M}_{j=1,j\neq i}p(j|i,u,v)\bigg{|}\big{(}\V[P(j)]-\V[Q(j)]\big{)}\bigg{|} \nonumber \\ 
&\hspace{0.3cm} +(1-w+w\alpha p(i|i,u,v))\bigg{|}\big{(}\V[P(i)]-\V[Q(i)]\big{)}\bigg{|} \label{apply-l1}\\
\leq & w \alpha \displaystyle\sum^{M}_{j=1,j\neq i}p(j|i,u,v)\displaystyle\max_{u,v}\bigg{|}P(j,u,v)-Q(j,u,v)\bigg{|} \nonumber \\ 
&\hspace{0.3cm} +(1-w+w\alpha p(i|i,u,v))\displaystyle\max_{u,v}\bigg{|}P(i,u,v)-Q(i,u,v)\bigg{|} \nonumber \\ 
\leq & (w \alpha + 1-w)\|P-Q\|. \label{apply2-l1} 
\end{align}
Since the RHS is not a function of $(i,u,v)$, we have
\begin{align}
&\displaystyle \max_{i,u,v}|(H_{w}P-H_{w}Q)(i,u,v)|\leq (w \alpha +1-w)\|P-Q\|, \nonumber \\
&\text{or }\|(H_{w}P-H_{w}Q)\|\leq (w \alpha+1-w)\|P-Q\|. \nonumber
\end{align}
Note the use of the assumption $0< w\leq w^*$ (with $w^*$ as in \eqref{w-star}) in equation \eqref{w-con} that ensures that the term $\big{(}1-w+w\alpha p(i|i,u,v)\big{)}\geq0$, to arrive at equation \eqref{apply-l1}. 
Also equation \eqref{apply2-l1} is obtained by an application of Lemma \ref{l1} in equation \eqref{apply-l1}.
From the assumptions on $w$ and discount factor $\alpha$, it is clear that $0\leq (w\alpha+1-w)<1$. 
Therefore $H_w$ is a max-norm contraction with contraction factor 
$(w\alpha+1-w)$ and $Q^{\dagger}$ is its unique fixed point.
\end{proof}

\begin{lemma}\label{c2}
$T_{w}$ is a contraction with contraction factor $(1-w+w\alpha).$
\end{lemma}
\begin{proof}
The proof is analogous to the proof of the Lemma \ref{l5}. 
\end{proof}
\begin{lemma}
\label{l6}
For $1\leq w \leq w^*$, the contraction factor for the map $H_w,$
\begin{align*}
1-w+\alpha w\leq \alpha.
\end{align*}
\end{lemma}
\begin{proof}
For $1\leq w \leq w^*$, define $f(w)=1-w+\alpha w.$ Let $w_1 < w_2.$ Then $f(w_1)=1-w_1(1-\alpha) > 1-w_2(1-\alpha)=f(w_2).$ Hence $f$ is decreasing. In particular, for $w \in [1,w^*]$, $1-w+\alpha w =f(w)\leq f(1)= \alpha.$ This shows that, if $w^* > 1$ and $w$ is chosen such that $1 < w \leq w^*$, the contraction factor is strictly smaller than $\alpha$.
\end{proof}
\begin{remark}\label{remark-cases}
Depending on the choice of $w$, the following observations can be made about our proposed generalized minimax Q-learning algorithm (refer Algorithm \ref{alg:minimax Q-learning}). 
\begin{itemize}
    \item \textbf{Case I ($w = 1$)~:} The generalized minimax Q-learning reduces to standard minimax Q-learning. 
    \item \textbf{Case II ($w < 1$)~:} 
    The contraction factor of $H_{w}$ in this case, $1-w+\alpha w > \alpha$, giving rise to minimax Q-learning algorithm with slower convergence.
    \item \textbf{Case III ($w > 1$)~:} For this choice of $w$, it is required that $p(i|i,u,v) > 0, ~ \forall (i,u,v)$ (refer equation \eqref{w-star}). Under this condition, as shown in Lemma \ref{l6}, the contraction factor of $H_w$, $(1-w+\alpha w)< \alpha$, giving rise to a faster minimax Q-learning algorithm.  
\end{itemize}
\end{remark}
\begin{lemma}
Let \textcolor{black}{$Q^{*}(i,u,v)= r(i,u,v)+\alpha\displaystyle\sum_{j \in S}p(j|i,u,v)J^{*}(j)$} and $Q^{\dagger}$ be the fixed point of $H_{w}$. Then for all $(i,u,v) \in S\times U \times V$, \textcolor{black}{$$Q^{\dagger}(i,u,v)-Q^{*}(i,u,v)=(1-w)\big{(}J^*(i)-Q^{*}(i,u,v)\big{)}.$$} Moreover, $\V[Q^{\dagger}(i)]=\V[Q^{*}(i)]$ $\forall i \in S$.
\label{l7}
\end{lemma}
\begin{proof}
By the hypothesis on $Q^{*}$, $\V[Q^{*}(i)]=TJ^{*}(i)=J^{*}(i) = T_{w}J^{*}(i)=\V[Q^{\dagger}(i)] ~ \forall i \in S $.
Since $Q^{\dagger}$ is the fixed point of $H_{w}$, we have
\begin{align*}
      & Q^{\dagger}(i,u,v)=(H_{w}Q^{\dagger})(i,u,v)\\ 
      &=  w\bigg{(}r(i,u,v)+\alpha \sum^{M}_{j=1}p(j|i,u,v)J^*(j)\bigg{)} +(1-w)J^*(i) 
\end{align*} Therefore
\begin{align*}
     Q^{\dagger}(i,u,v)-Q^*(i,u,v) &=(1-w)\big{(}J^*(i)-Q^*(i,u,v)\big{)}.
\end{align*}
This completes the proof. 
\end{proof}
\textcolor{black}{This Lemma is an interesting and important result in our paper. It shows that, even if the standard minimax Q-value iterates and generalized minimax Q-value iterates are not the same for all $(i,u,v)$ tuples, the min-max values at each state given by both the algorithms are equal. Therefore, this lemma states that generalized minimax Q-value iteration computes the min-max value function, which is the goal of the two-player zero-sum Markov game.}
We now show the convergence of generalized minimax Q-learning (refer Algorithm \ref{alg:minimax Q-learning}). For this purpose, we first state the following result (Proposition 4.5 on page 157 of \cite{bertsekas1996neuro}) and apply it to show the convergence of our proposed algorithm. We consider $\g_n(i)$ to be deterministic as with our algorithm, unlike \cite{bertsekas1996neuro} where these are allowed to be random.
\begin{theorem}\label{stoc-fp}
Let $\{r_{n}\}$ be the sequence generated by the iteration $$r_{n+1}(i)=(1-\g_{n}(i))r_{n}(i)+\g_{n}(i)\big{(}Fr_{n}(i)+N_{n}(i)\big{)}, n\geq0.$$
\begin{itemize}
    \item The step-sizes $\g_{n}(i)$ are non-negative and satisfy 
    $$\displaystyle\sum^{\infty}_{n=0}\g_{n}(i)=\infty, \displaystyle\sum^{\infty}_{n=0}\g^2_{n}(i)<\infty.$$
    \item The noise terms $N_{n}(i)$ satisfy
    \begin{itemize}
        \item For every $i$ and $n$, we have $\E[N_{n}(i)|\F_{n}]=0,$ where
    \begin{align*}
        \F_{n} & =\sigma\big{\{}r_0(i),...,r_{n}(i),
        N_{0}(i),\cdots,N_{n-1}(i),\\
        & \hspace{5cm} 1\leq i \leq d \big{\}}.
    \end{align*}
        \item Given any norm $\|.\|$ on $\R^{d}$, there exist constants $C$ and $D$ such that 
        $$\E[N_{n}^2(i)|\F_{n}]\leq C+D\|r_{n}\|^2, \forall i,n.$$
    \end{itemize}
    \item The mapping $F:\R^{d}\rightarrow \R^{d} $ is a max-norm contraction.
\end{itemize}
Then, $r_{n}$ converges to $r^*,$ the unique fixed point of $F$, with probability 1.
\end{theorem}

\begin{theorem}\label{thm1}
Given a finite state-action two-player zero-sum Markov game $(S,U,V,p,r,\alpha)$ with bounded payoffs i.e. $|r(i,u,v)| \leq R<\infty, ~ \forall ~ (i,u,v) \in S\times U \times V $, the generalized minimax Q-learning algorithm (see Algorithm \ref{alg:minimax Q-learning}) given by the update rule:
\begin{align*}
& Q_{n+1}(i,u,v) =  Q_{n}(i,u,v) \\
& \hspace{1cm} +\g_{n}\Bigg{(}w \Big{(}r(i,u,v)  +\alpha \V[Q_{n}(Y_{n}(i,u,v))]\Big{)}\\
& \hspace{3cm}+(1-w)\V[Q_{n}(i)]-Q_n(i,u,v)\Bigg{)}
\end{align*}
converges with probability 1 to $Q^{\dagger}(i,u,v)$ as long as
$$\sum_n \g_n= \infty, \hspace{1cm} \sum_n \g^2_n<\infty,$$
for all $(i,u,v) \in S\times U \times V$.
\end{theorem}
\begin{proof}
The update rule of the algorithm is given by
\begin{align*}
    & Q_{n+1}(i,u,v)=\big(1-\g_n\big)Q_{n}(i,u,v) \\
    & \hspace{2cm} +\g_n\Big{[}w\big{(}r(i,u,v)+\alpha \V[Q_{n}(Y_{n}(i,u,v))]\big{)} \\
    & \hspace{3cm}+(1-w)\V[Q_{n}(i)]\Big{]}.
\end{align*}
Let $\F_n=\sigma\big{(}\{Q_{0}, Y_{j}, \forall j < n\}\big{)},n \geq 0$ be the associated filtration. Now observe that $Y_{n}(i,u,v) \sim p(.|i,u,v).$ Also, given $(i,u,v)$, assume that the random variables $Y_{n}(i,u,v), n \geq 0$ are independent. 
Then the above equation can be rewritten as:
\begin{align}
    Q_{n+1}(i,u,v) =& \big(1-\g_n\big)Q_{n}(i,u,v) + \nonumber \\ & \g_n\Big{(}(H_wQ_n)(i,u,v)+ N_n(i,u,v)\Big{)}, 
\end{align}
where 
\begin{align}
    (H_wQ_n)(i,u,v) &= \E \Big{[}w\big{(}r(i,u,v)+\alpha \V[Q_{n}(Y_{n}(i,u,v))]\big{)} \nonumber \\ & +(1-w)\V[Q_{n}(i)] \Big{|} \F_n  \Big{]},
\end{align} and
\begin{align}
    N_n(i,u,v) &= \Big{(} w\big{(}r(i,u,v)+\alpha \V[Q_{n}(Y_{n}(i,u,v))]\big{)}  \nonumber \\ & +(1-w)\V[Q_{n}(i)]\Big{)} \nonumber \\ & -\E \Big{[}w\big{(}r(i,u,v)+\alpha \V[Q_{n}(Y_{n}(i,u,v))]\big{)} \nonumber \\ & +(1-w)\V[Q_{n}(i)]  \Big{|} \F_n \Big{]}.
\end{align}
Now note, from Lemma \ref{l5}, that the mapping $H_w$ is a max-norm contraction. \textcolor{black}{Also, by the definition of $N_n$, we have that $N_n$ is $\F_{n+1}-$measurable $\forall n$}. Further,
\begin{align}\label{martingaleDiff}
    \E[N_n\Big{|}\F_n] = 0, ~ \forall n.
\end{align}
\textcolor{black}{Finally, as $Y_{n}$ is independent of $\F_{n}$, we have}
\begin{align}
  & \E\big{[}N^2_n(i,u,v) \Big{|} \F_{n}\big{]} \nonumber \\
= ~ &\E \Bigg[ \bigg(w \big(r(i,u,v)+\alpha \V[Q_{n}(Y_{n}(i,u,v))]\big{)} \nonumber \\
& \hspace{2cm} +(1-w) \big(\V[Q_{n}(i)]\big)-H_{w}Q_{n}(i,u,v)\bigg)^2 \Bigg] \nonumber \\
\leq ~ &\E \Bigg[ \bigg(w \big(r(i,u,v)+\alpha \V[Q_{n}(Y_{n}(i,u,v))]\big) \nonumber \\
& \hspace{3cm} +(1-w) \big(\V[Q_{n}(i)]\big)\bigg)^2 \Bigg] \nonumber \\
\leq ~ & 3\bigg{(}w^2R^2+\alpha^2w^2\|Q_n\|^2+(1-w)^2\|Q_n\|^2\bigg{)} \nonumber \\
=& (C+D ~ \|Q_n\|^2), \label{MDSineq}
\end{align}
where $C=3w^2R^2$ and $D=3\big{(} \alpha^2w^2+(1-w)^2 \big{)}.$
Here the first inequality follows from the fact:
\begin{align*}
\E[Z-\E Z]^2=\E[Z^2]-\E[Z]^2 \leq \E[Z^2].    
\end{align*}
The second inequality follows from the following facts:
 \begin{align*}
     & |r(i,u,v)| \leq R, \\
     & \|v\|= \max_{i}{|v(i)|},\\
     & (a+b+c)^2\leq3(a^2+b^2+c^2) ~ \forall a,b,c ~ \text{and} \\
     & \text{Corollary}~ \ref{l2}.
 \end{align*} 
\textcolor{black}{Therefore by Theorem 1, with probability 1, the generalized minimax Q-learning iterates $Q_n$ converge. By virtue of Lemma \ref{l7}, our proposed minimax Q-learning algorithm computes a policy whose value is the min-max value of the Markov game.}
\end{proof}
\subsection{Extension to the asynchronous setting}
In the setting considered above, the updates are synchronous, i.e., Q-values of all state-action pairs are updated at every iteration. However, in the case of online settings, only a single sample is obtained through the interaction with the environment. In the following, we describe the convergence of our algorithm in the asynchronous settings. The following assumption on the structure of probability transition matrix $p$ and the control policies \cite[Page 130]{borkar2009stochastic} is necessary in the asynchronous setting:
\begin{assumption}\label{ergodicassump}
The Markov chain induced by all the control policies is ergodic. Moreover, under each policy, every action can be picked with a positive probability in any state. 
\end{assumption}

The latter requirement in Assumption \ref{ergodicassump} is satisfied for instance by policies such as $\epsilon-$greedy, see \cite{sutton1998introduction}. We first state a result from \cite[Theorem 3]{tsitsiklis1994asynchronous} and apply it to show the convergence of our proposed algorithm. 
Let $T^i$ be an infinite subset of $\mathcal{N}$ and let $\{r_{n}\} \in \mathcal{R}^m$ be the sequence generated by the iteration $$r_{n+1}(i)=
\begin{cases}
r_n(i), ~ &n \not\in T^i, \\
(1-\g_{n}(i))r_{n}(i)+\g_{n}(i)\big{(}Fr^i_{n}(i)+N_{n}(i)\big{)},~ &n \in T^i.
\end{cases}
$$
Here, $r_n(i)$ is a vector of possibly outdated components of $r$. In particular, we let
\begin{align*}
    r^i_n = (r_{\tau_1^i(n)}(1),\ldots,r_{\tau_m^i(n)}(m)),
\end{align*}
where each $\tau^i_j(n)$ is an integer satisfying $0 \leq \tau^i_j(n) \leq n$ representing the delay in information about component $j$ available while updating component $i$ at time $n$. If $\tau^i_j(n) = n, ~ \forall i,j$ then this reduces to the synchronous setting. 

Let $\F_{n} =\sigma\big{\{}r_0(i),\cdots,r_{n}(i), \gamma_0(i), \cdots, \gamma_n(i), \tau^i_j(0),\cdots,\tau^i_j(n), \\
        N_{0}(i),\cdots,N_{n-1}(i),
        ~ 1\leq i,j \leq m \big{\}}$. It is important to note from the construction of $\{\F_n\}$ that, the step-size sequences $\gamma_n(i)$ are in general allowed to be random. Thus, the component to be updated at time $n$ can be decided online based on the history until time $n$. \begin{assumption}\label{newa1}
 For any $i$ and $j$, $\lim_{n \rightarrow \infty} \tau^i_j(n) = \infty,$ with probability 1. 
 \end{assumption}
 \begin{assumption}\label{newa2}
 For every $i$ and $n$, $N_n(i)$ is $\F_{n+1}-$measurable and $\E[N_{n}(i)|\F_{n}]=0$.
 \end{assumption}
 \begin{assumption}\label{newa3}
 $\E[N_{n}^2(i)|\F_{n}]\leq C+D \max_j \max_{\tau \leq n}|r_{\tau}(j)|^2, ~ \forall i,n.$
 \end{assumption}
\begin{assumption}\label{newa4}
The step-sizes $\g_{n}(i)$ are non-negative and satisfy 
    $$\displaystyle\sum^{\infty}_{n=0}\g_{n}(i)=\infty, \displaystyle\sum^{\infty}_{n=0}\g^2_{n}(i)<\infty, ~ \text{with probability}~ 1. $$
\end{assumption}
\begin{assumption}\label{newa5}
There exists a vector $r^*$, a positive vector $v$, a scalar $\beta \in [0,1)$, such that
\begin{align*}
    \|F(r) - r^*\|_{v} \leq \beta \|r - r^*\|_{v}, \forall r \in \R^m.
\end{align*}
\end{assumption}
\begin{theorem}\label{stoc-fp-async}
Under Assumptions \ref{newa1}-\ref{newa5}, $r_n$ converges to $r^*,$ the unique fixed point of $F$, with probability 1.
\end{theorem}
\begin{theorem}\label{async-thm1}
Consider a finite state-action two-player zero-sum Markov game $(S,U,V,p,r,\alpha)$ with bounded payoffs i.e. $|r(i,u,v)| \leq R<\infty, ~ \forall ~ (i,u,v) \in S\times U \times V $. Let the sample at iteration $n$ be $(i_n,u_n,v_n, Y_n(i_n,u_n,v_n))$. Then, under Assumption \ref{ergodicassump}, the asynchronous generalized minimax Q-learning algorithm given by the update rule:
\begin{align*}
Q_{n+1}(i,u,v) =  \begin{cases}
Q_n(i,u,v), ~ \text{if}~ (i,u,v) \neq (i_n,u_n,v_n) \\
Q_{n}(i,u,v) +\g_{n}(i,u,v)\Bigg{(}w \Big{(}r(i,u,v)  +\alpha \V[Q_{n}(Y_{n}(i,u,v))]\Big{)}+ \\(1-w)\V[Q_{n}(i)]-Q_n(i,u,v)\Bigg{)}, ~ \text{if} ~(i,u,v) = (i_n,u_n,v_n)
\end{cases}
\end{align*}
converges with probability 1 to $Q^{\dagger}(i,u,v)$ for all $(i,u,v) \in S\times U \times V$.
\end{theorem}
\begin{proof}
Assumption \ref{newa1} is trivially satisfied as there is no delay in information during the training. Hence $\tau_j^i(n) = n, ~ \forall i,j$. Assumptions \ref{newa2} and \ref{newa3} are shown in \eqref{martingaleDiff} and \eqref{MDSineq}, respectively. In order for Assumption \ref{newa4} to be true, all state and action pairs have to be visited infinitely often, which is ensured through Assumption \ref{ergodicassump}. 
Finally, from Lemma \ref{l5},
\begin{align*}
    \|H_w(Q) - H_w(Q^\dagger)\| \leq (w\alpha + 1 -w) \|Q - Q^\dagger\|,~ \forall Q
\end{align*}
However, as $Q^\dagger$ is the unique fixed point, we have,
\begin{align*}
    \|H_w(Q) - Q^\dagger\| \leq (w\alpha + 1 -w) \|Q - Q^\dagger\|,~ \forall Q,
\end{align*}
thereby proving Assumption \ref{newa5}.
Therefore, by Theorem \ref{stoc-fp-async}, with probability 1, the asynchronous generalized minimax Q-learning iterates $Q_n$ converge to $Q^\dagger$.
\end{proof}


\section{Relation between Generalized Minimax Q-learning and standard Minimax Q-learning}
\color{black}
In this section, we describe the relation between our proposed Generalized Minimax Q-learning and standard Minimax Q-learning algorithms. For the given two-player zero-sum Markov game $(S,U,V,p,r,\alpha)$, we construct a new game   $(\bar{S},\bar{U},\bar{V},q,\bar{r},\bar{\alpha})$ as follows: 
\begin{itemize}
    \item $\bar{S}=S, \bar{U}=U,\bar{V}=V $
    \item $\bar{r}=wr, \bar{\alpha}=(1-w+\alpha w)$  and
    for a given $(i,u,v)$, let $q(.|i,u,v):S\rightarrow [0,1]$ be defined as \[ q(k|i,u,v)= \begin{cases} 
      \frac{w\alpha p(k|i,u,v)}{(1-w+w\alpha)}, ~ k\neq i,\\
      \frac{1-w+w\alpha p(i|i,u,v)}{(1-w+w\alpha)}, ~  k=i,
   \end{cases}\]
    where $0<w\leq w^*$. We note that $q(.|i,u,v)$ is a probability mass function on $\bar{S}$.
\end{itemize}
\textcolor{black}{Now consider the standard minimax $Q$-Bellman operator $\bar{H}$ for this game given by, $\bar{H}: \R^{\bar{S} \times \bar{U} \times \bar{V}} \rightarrow \R^{\bar{S} \times \bar{U} \times \bar{V}}$ and $\bar{H}Q(i,u,v)=w r(i,u,v)+(1-w+\alpha w)\displaystyle\sum_{j \in \bar{S}}q(j|i,u,v)\V[Q(j)]$, where $Q(j)$ is $|\bar{U}|\times |\bar{V}|$ dimensional matrix with $(u,v)^{\text{th}}$ entry as $Q(j,u,v).$ and $\V[Q(j)]$ is given by the equation \eqref{v-def}}.
Note that 
\begin{align*}
    \bar{H}Q(i,u,v)&=w r(i,u,v)+(1-w+\alpha w)\displaystyle\sum_{j \in \bar{S}}q(j|i,u,v)\V[Q(j)] \\
    &=w r(i,u,v)+\displaystyle\sum_{j \in \bar{S}, j \neq i }w\alpha p(j|i,u,v)\V[Q(j)] \\
    & \hspace{1cm}+(1-w+w\alpha p(i|i,u,v))\V[Q(i)] \\
    &=w\left(r(i,u,v)+\alpha\displaystyle\sum_{j \in S}p(j|i,u,v)\V[Q(j)]\right)\\
    & \hspace{2cm}+(1-w)\V[Q(i)]\\
    &=H_{w}Q(i,u,v).
\end{align*}
Hence $\bar{H}$ operator of the game $(\bar{S},\bar{U},\bar{V},q,\bar{r},\bar{\alpha})$ is same as the $H_{w}$ operator defined for the game $(S,U,V,p,r,\alpha)$.
Let us consider an iteration of the minimax $Q$-learning algorithm on $(\bar{S},\bar{U},\bar{V},q,\bar{r},\bar{\alpha})$ given by
\begin{align*}
    \bar{Q}_{n+1}(i,u,v)
    =&\big(1-\g_n\big)\bar{Q}_{n}(i,u,v) +\g_n\big{(}w r(i,u,v)+(1-w+w\alpha) \V[\bar{Q}_{n}(\bar{Y}_{n}(i,u,v))]\big{)}\\
\hspace{0.5cm}  =& \big(1-\g_n\big)\bar{Q}_{n}(i,u,v)+\g_n\Big{(}(\bar{H}\bar{Q}_n)(i,u,v)+ \bar{N}_n(i,u,v)\Big{)}, 
\end{align*}
where $\g_n, ~ n \geq 0,$ is the step-size sequence, $\bar{Y}_n(i,u,v) \sim q(.|i,u,v)$, $\bar{N}_n(i,u,v)=\Big{(}w r(i,u,v)+(1-w+w\alpha)\V[\bar{Q}_{n}(\bar{Y}_{n}(i,u,v))]\Big{)}  -\bar{H}\bar{Q}_{n}(i,u,v)$ and compare it with an iteration of Generalized minimax Q-learning. Since $\bar{H}=H_{w}$, both algorithms converge to $Q^\dagger$, the fixed point of $H_{w}$, and differ only in the per-iterate noise $\bar{N}_{n}$ and $N_{n}$.
\textcolor{black}{
\begin{lemma}\label{relationlemma}
Suppose $\{Q_{n}\}$ are the iterates of Generalized minimax Q-learning. Then given any $\epsilon >0$ there exists a natural number $N$ that is possibly sample path dependent,  such that $\|Q_{n}\|\leq \frac{R}{1-\alpha} + \epsilon$, for $n>N$ almost surely.
\end{lemma}}
\begin{proof}
Consider the iterates $\bar{Q}_{n}$ of the minimax Q-learning algorithm with respect to the stochastic game $(\bar{S},\bar{U},\bar{V},q,\bar{r},\bar{\alpha})$ with initial point $\|\bar{Q}_{0}\| \leq \frac{R}{1-\alpha}$. Now assume that $\|\bar{Q}_{n}\|\leq \frac{R}{1-\alpha}$ (induction hypothesis). Then
\begin{align*}
\|\bar{Q}_{n+1}\| & \leq (1-\g_{n})\|\bar{Q}_{n}\|+\g_{n}(w\|r\|+(1-w+w\alpha)\|\bar{Q}_{n}\|)\\
&\leq (1-\g_{n}) \frac{R}{1-\alpha} + \g_{n}\left(w R+ (1-w+w\alpha)\frac{R}{1-\alpha}\right) =\frac{R}{1-\alpha}.
\end{align*}
Therefore by induction $\|\bar{Q}_{n}\| \leq \frac{R}{1-\alpha}, \forall n \geq 0$. 
As the sequences $\{\bar{Q}_{n}\}$ and $\{Q_{n}\}$ converge to $Q^\dagger$, given $\epsilon >0$ there exists a natural number $N$ such that $\|Q_{n}-\bar{Q}_{n}\|\leq \epsilon \implies \|Q_{n}\| \leq \frac{R}{1-\alpha} + \epsilon, \forall n > N$. Moreover $\|Q^\dagger\|\leq \frac{R}{1-\alpha}$. To conclude we have $\|Q_{n}\| \leq \frac{R}{1-\alpha} +\epsilon$ almost surely and $N$ here is possibly sample path dependent. This completes the proof.
\end{proof}
\begin{remark}
We invoke the standard Q-learning algorithm on $(\bar{S},\bar{U},\bar{V},q,\bar{r},\bar{\alpha})$ with the initial point $\bar{Q}_{0}$ chosen such that $\|\bar{Q}_{0}\| \leq \frac{R}{1-\alpha}$ to prove the Lemma \ref{relationlemma}. It is also possible to obtain the same desired conclusion by directly utilizing the convergence of the iterates of standard Q-learning algorithm on $(\bar{S},\bar{U},\bar{V},q,\bar{r},\bar{\alpha})$ for any arbitrary initial point $\bar{Q}_{0}$. 
\end{remark}

\section{Model-free Generalised Minimax Q-learning} \label{modimpalg}
\color{black}
Note that an input to the Algorithm \ref{alg:minimax Q-learning} is the relaxation parameter $w \leq w^*$, where $w^*$ is defined in \eqref{w-star}. As $w^*$ depends on the transition probability function $p$, it is not possible to choose a valid $w$ in the experiments, where we do not have access to probability transition function. \textcolor{black}{In this section, we describe a synchronous version of the model-free generalised minimax Q-learning procedure that mitigates the dependency on the model information.}

\textcolor{black}{We maintain a count value $C_n[i][j][u][v], ~ \forall i,j \in S, u \in U, v \in V$ (initialised to zero $~ \forall i,j,u,v$) that represents the number of times the sample $(i,u,v,j)$ has been encountered until iteration $n$. We define
\begin{align}\label{transprob}
    p'_n(j|i,u,v) = \frac{C_n[i][j][u][v]}{n}, ~ \forall n \geq 1,
\end{align}
with $p'_0(j|i,u,v) = 0, ~ \forall i,j,u,v$.
It is easy to see that
\begin{align}
\label{transprobconv}
    p'_n(j|i,u,v) \xrightarrow[]{} p(j|i,u,v), ~ \forall i,j,u,v,
\end{align}
as $n \xrightarrow[]{} \infty$, almost surely (from the Strong Law of Large Numbers).}

Now, we propose our model-free ``generalised minimax Q-learning'' by modifying the Step 3 of Algorithm \ref{alg:minimax Q-learning} as:
\begin{align}
    d_{n+1}(i,u,v) = w_n\Big(r(i,u,v)+\alpha\V [Q_n(Y_n(i,u,v))]\Big)+(1-w_n)\V[Q_n(i)],
\end{align}
where the sequence $\{w_n,~ n \geq 1\}$ is updated as:
\begin{align}\label{wnrecursion}
    w_{n+1} = (1-\gamma_n)w_n + \gamma_n \Big(\frac{1}{1 - \alpha \displaystyle \min_{i,u,v}p'_n(i|i,u,v)}  \Big),
\end{align}
with $w_0 \in [1,\frac{1}{1-\alpha}]$.
\subsection{Convergence Analysis:}
We write the two update equations as follows:
\begin{align}
\label{ttup1}
    Q_{n+1}(i,u,v) = Q_n(i,u,v) + \gamma_n (&h(Q_n,w_n)+ M_{n+1}), ~ \forall i,u,v, \\
\label{ttup2}
    w_{n+1} = w_n + \gamma_n (g(w_n) + \epsilon_n).
\end{align}
The function $h$ is defined as: 
\begin{align*} 
h(Q_n,w_n)(i,u,v) &= H_{w_n}(Q_n)(i,u,v) - Q_n(i,u,v) \\ &= \E\Big[w_n\Big(r(i,u,v)+\alpha\V [Q_n(j)]\Big)
+(1-w_n)\V[Q_n(i)] - Q_n(i,u,v)\Big].
\end{align*}
The sequence $\{M_n\}$ defined as:
\begin{align*}
    M_{n+1} = w_n\Big(r(i,u,v)+\alpha\V [Q_n(Y_n(i,u,v))]\Big)
+(1-w_n)\V[Q_n(i)] - Q_n(i,u,v) - h(Q_n,w_n),
\end{align*}
is a martingale difference noise sequence with respect to the increasing $\sigma-$fields \\
$\F_n := \{Q_0,w_0,M_0,\ldots,w_n,M_n\}$, $ n \geq 0,$ satisfying
\begin{align}\label{cond1}
    \E[\|M_{n+1}\|^2|\F_n] &\leq 3w_n^2\Big(\displaystyle \max_{i,u,v}|r(i,u,v)| \Big)^2 + \frac{6\alpha^2}{(1-\alpha)^2}\|Q_n\|^2 \nonumber \\ 
                            &\leq K \big(1+\|w_n\|^2+\|Q_n\|^2\big),
\end{align}
where $K = \max \Big\{3\Big(\displaystyle \max_{i,u,v}|r(i,u,v)| \Big)^2,\frac{6\alpha^2}{(1-\alpha)^2} \Big\}$.
The function $g$ is defined as:
\begin{align*}
    g(w_n) = \displaystyle \frac{1}{1 - \alpha \displaystyle \min_{i,u,v}p(i|i,u,v)} - w_n.
\end{align*}
Finally, $\epsilon_n = 
    \frac{1}{1 - \alpha \displaystyle \min_{i,u,v}p'_n(i|i,u,v)} - \frac{1}{1 - \alpha \displaystyle \min_{i,u,v}p(i|i,u,v)},$
where $p'_n$ is updated as shown in the equation \eqref{transprob}. Note that, from \eqref{transprobconv}, we get 
\begin{align}\label{cond2}
    \epsilon_n \xrightarrow[]{} 0, ~ \text{as} ~ n \xrightarrow[]{} \infty, ~ \text{almost surely. }
\end{align}
Notice from \eqref{ttup1}-\eqref{ttup2} that the $Q_n$-recursion in \eqref{ttup1} depends on the $w_n$-update in \eqref{ttup2}, while the latter is an independent update that does not depend on $Q_n$. Let $Q_{w^*}^{\dagger}$ be the (unique) fixed point of $H_{w^*}$. Note that, from \eqref{wnrecursion}, $w_n \in \Big[1,\displaystyle \frac{1}{1-\alpha}\Big], ~ \forall n \geq 0$. \textcolor{black}{Therefore, $\{w_n, ~ \forall n \geq 1\}$ updates are bounded.} We now make an assumption on the boundedness of $\{Q_n\}$ iterates. 
\begin{assumption}\label{assump1}
$\|Q_n\| \leq B < \infty,~ \forall n \geq 0$.
\end{assumption}
In practice, the iterates $\{Q_n\}$ will satisfy Assumption \ref{assump1} if they are projected to a prescribed compact set $\Omega$ whenever they exit it, see for instance, \cite[Chapter 5]{kushner2012stochastic} for a general setting of projected stochastic approximation. From Lemma \ref{relationlemma}, the solution $\|Q^{\dagger}_{w^*}\| \leq \frac{R}{1-\alpha}$. Therefore, we can choose the set $\Omega$ such that $\|\Omega\| := \displaystyle \max_{x \in \Omega} \|x\| > \frac{R}{1-\alpha}$.
\begin{lemma}\label{lipschitzP}
Functions $h(Q,w)$ and $g(w)$ are Lipschitz.
\end{lemma}
\begin{proof}
Consider $p,q \in [-B,B]^{|S|\times |U| \times |V|}$ and $w_1,w_2 \in [1,\frac{1}{1-\alpha}]$. Let $R = \displaystyle \max_{i,u,v}|r(i,u,v)|$. Then,
\begin{align*}
    \|h(p,w_1) - h(q,w_1)\| &\leq  \|H_{w_1}(p) - H_{w_1}(q)\| + \|p-q\| \leq (|1-w_1|+w_1\alpha)\|p-q\| + \|p-q\| \\ & \leq \left(\frac{1+\alpha}{1-\alpha}\right)\|p-q\|. \\
    |(h(q,w_1) - h(q,w_2))(i,u,v)| &\leq |w_1 - w_2| \E[|r(i,u,v)+ \alpha \V[q(j)] - \V[q(i)|]
     \leq |w_1 - w_2|(R + 2B).
\end{align*}
Hence, $\|h(p,w_1) - h(q,w_2)\| \leq L(\|p-q\| + |w_1 -w_2|)$,
where $L= \displaystyle \max\left\{\frac{1+\alpha}{1-\alpha},R + 2B\right\}$. 
Finally, $|g(w_1) - g(w_2)| \leq |w_1 - w_2|$.
Therefore the functions $h(Q,w)$ and $g(w)$ are Lipschitz. 
\end{proof}
We now consider the iterates \eqref{ttup1}-\eqref{ttup2} in a combined form as follows:
\begin{align}\label{combiter}
    x_{n+1} = x_n + \gamma_n(f(x_n)+M'_{n+1}+ \epsilon'_n), ~ \text{where} 
\end{align}
$x_n =\left(Q_{n},w_{n}\right)^{T}
$, 
$f(x_n) = \left(H_{w_n}(Q_n) - Q_n, w^* - w_n\right)^{T}
$, 
$M'_{n+1} =\left(M_{n+1}, 0 \right)^{T},
$
$\epsilon'_n = \left(0,\epsilon_n\right)^{T}
$.
\textcolor{black}{Let $Q^\dagger_{w^*}$ be the fixed point of the modified min-max Q-Bellman operator (see \eqref{MQU}) when $w = w^*$ is used.}
\begin{theorem}
\label{twol1}
$x_n \xrightarrow[]{} x^*$, where $x^* = \left(Q^{\dagger}_{w^*},w^* \right)^{T}
$, almost surely. 
\end{theorem}
\begin{proof}
\textcolor{black}{The iterates $\{x_n\}$ in \eqref{combiter} track the ODE \cite[Section 2.2]{borkar2009stochastic} ~  
$\dot x = f(x) = \left(H_{w}(Q) - Q,w^* - w\right)^{T}
$.
Note that $w_n$ iterates drive the $Q_n$ iterates but the reverse is not true, i.e, it is a one way coupling of the dynamics. First, consider the ODE $\dot w = w^* - w$. Let $g_\infty(w) = \displaystyle \lim_{r \xrightarrow[]{} \infty} \frac{g(rw)}{r}$. The function $g_\infty(w)$ exists and is equal to $-w$. Moreover, the origin is the unique globally asymptotically stable equilibrium for the ODE
\begin{align*}
    \dot w = g_\infty(w) = -w,
\end{align*}
with $V(w) = \frac{w^2}{2}$ serving as an associated Lyapunov function.  Further, $w^*$ is the unique globally asymptotically stable equilibrium for the ODE $\dot w = w^* - w$. Therefore, by \cite[Theorem 7, Chapter 3 and Theorem 2 - Corollary 4]{borkar2009stochastic}, we have $w_n \xrightarrow[]{} w^*$ almost surely. 
The $Q_n$ iterates now track the ODE given by $\dot Q = H_{w^*}(Q) -Q$. By virtue of Lemma \ref{l5}, $H_w^*$ is a contraction. Hence, by Stochastic fixed point analysis \cite[Section 10.3]{borkar2009stochastic}, we have $Q_n \xrightarrow{} Q^\dagger_{w^*}$, almost surely.}
\end{proof}

\begin{remark}
One way to prove the Assumption \ref{assump1} is to project the $\{Q_n\}$ iterates onto a prespecified convex and compact set $C$. Under projection, the update equation (eq. \eqref{ttup1}), i.e.,
\begin{align}
\label{ttup1-sa}
    Q_{n+1}(i,u,v) = Q_n(i,u,v) + \gamma_n (&h(Q_n,w_n)+ M_{n+1}), ~ \forall i,u,v,
\end{align}
is replaced with
\begin{align}
\label{ttup1-proj}
    Q_{n+1}(i,u,v) = \Gamma_{C} \Big\{ Q_n(i,u,v) + \gamma_n (&h(Q_n,w_n)+ M_{n+1})\Big\}, ~ \forall i,u,v,
\end{align}
where $\Gamma_{C}(P)$ is the projection of $P$ onto a compact and convex set such as $C = [-B,B]^{|S|\times |U| \times |V|}$. Convexity of $C$ would ensure that $\Gamma_{C}(P)$ is a unique fixed point for any $P$.

The iterates $\{Q_n\}$ in \eqref{ttup1-proj} track the ODE \cite[Chapter 5]{kushner2012stochastic} \begin{align}\label{projSA}
    \dot Q = \hat\Gamma (H_{w^*}(Q) -Q),
\end{align}
where the operator $\hat\Gamma (h),$ for a continuous function $h$ is defined as:
\begin{align}
    \hat\Gamma (h(Q,w)) = \displaystyle \lim_{0 <\Delta \xrightarrow[]{} 0}\frac{[\Gamma_{C}(Q+\Delta h(Q,w)) - Q]}{\Delta},
\end{align}
From \cite[Theorem 2, Chapter 2]{borkar2009stochastic}, $\{Q_n\}$ iterates converge to a compact, connected, internally chain transitive, invariant set of the ODE \eqref{projSA}. It is easy to see that $\{Q^\dagger_{w^*}\}$ is an invariant and internally chain transitive (ICT) set of the ODE \eqref{projSA}. However, the projection operation will introduce spurious fixed points on the boundary of the set $C$ that will also be invariant and ICT sets of the ODE \eqref{projSA}. In \cite[Chapter 5.4]{borkar2009stochastic}, some practical techniques are discussed to avoid convergence to undesired equilibrium points (boundary points in this case). 
\end{remark}

\section{Experiments and Results}
We refer to the Algorithm \ref{alg:minimax Q-learning}, with $w = w^*$, as ``Generalised optimal minimax Q-learning'' and the model-free algorithm derived in the previous section as ``Generalised minimax Q-learning" algorithm in the experiments. We generate a two-player zero-sum Markov game and run all the algorithms for $50$ independent episodes in each of the three cases - (a). $10$ states and $5$ actions for each of the agents, (b). $20$ states and $5$ actions for each of the agents and (c). $50$ states and $5$ actions for each of the agents. The discount factor is set to $0.6$. The probability transition matrix generated satisfies $p(i|i,u,v) >0 ~ \forall i,u,v$ as this condition is required for faster performance of the generalized optimal minimax Q-learning and generalized minimax Q-learning. All the algorithms are run for $1000$ iterations in each episode with the same step-size sequences. 
\begin{table}[h!]
\begin{center}
\begin{tabular}{|c|c|c|c|}
\hline
\textbf{Algorithm}                                                                        & \textbf{10 states}                                                 & \textbf{20 states}                                                 & \textbf{50 states}                                                 \\ \hline
\textbf{\begin{tabular}[c]{@{}c@{}}Standard minimax\\ Q-learning\end{tabular}}            & \begin{tabular}[c]{@{}c@{}}0.68 $\pm$\\ 0.07\end{tabular}          & \begin{tabular}[c]{@{}c@{}}1.67 $\pm$\\ 0.13\end{tabular}          & \begin{tabular}[c]{@{}c@{}}3.99 $\pm$\\ 0.11\end{tabular}          \\ \hline
\textbf{\begin{tabular}[c]{@{}c@{}}Generalized minimax\\ Q-learning\end{tabular}}         & \textbf{\begin{tabular}[c]{@{}c@{}}0.49 $\pm$\\ 0.08\end{tabular}} & \textbf{\begin{tabular}[c]{@{}c@{}}1.43 $\pm$\\ 0.18\end{tabular}} & \textbf{\begin{tabular}[c]{@{}c@{}}3.75 $\pm$\\ 0.12\end{tabular}} \\ \hline
\textbf{\begin{tabular}[c]{@{}c@{}}Generalized optimal\\ minimax Q-learning\end{tabular}} & \begin{tabular}[c]{@{}c@{}}0.35 $\pm$\\ 0.08\end{tabular}          & \begin{tabular}[c]{@{}c@{}}1.26 $\pm$\\ 0.19\end{tabular}          & \begin{tabular}[c]{@{}c@{}}3.59 $\pm$\\ 0.14\end{tabular}          \\ \hline
\end{tabular}
\end{center}
\caption{Comparison of Error among three algorithms averaged across $50$ episodes}
\label{tab-exp}
\end{table}

The comparison criterion considered is the average error that is calculated as follows. At the end of each episode of the algorithm, the norm difference between estimate of the min-max value function and the actual min-max value function is computed. This process is repeated for all the $50$ episodes and the average is computed. Thus,
\begin{align}
    \text{Average Error} = \frac{1}{50}\displaystyle\sum_{k = 1}^{50} \|J^{*} - \V[Q_{k}(.)] \|_{2},
\end{align}
where $J^*$ is the min-max value function of the game and $Q_k(.)$ is the minimax Q-value function estimate obtained at the end of $k^{th}$ episode. 

In Table \ref{tab-exp}, we report the average error of three algorithms. We can see that, generalized optimal minimax Q-learning has the least average error, followed by the generalized minimax Q-leaning algorithm. This is expected as the generalized optimal Q-learning algorithm makes use of the optimal relaxation parameter $w^*$ in its updates, which is not practically feasible. 
\textcolor{black}{Therefore, we conclude that our proposed generalized minimax Q-learning algorithms perform empirically better (in terms of number of samples) than the standard minimax Q-learning algorithm.}  


\section{Conclusions}
In this work, we use the technique of successive relaxation to propose a modified min-max Bellman operator for two-player zero-sum games. We prove that the contraction factor of this modified min-max Bellman operator is less than the discount factor (contraction of the standard min-max Bellman operator) for the choice of $w>1$. The construction of the modified Q-Bellman operator enabled us to develop a generalized minimax Q-learning algorithm. We show the almost sure convergence of our proposed algorithm. We then derive a relation between our proposed algorithm and the standard minimax Q-learning algorithm. We also propose a model-free (from samples) version of our algorithm and prove its convergence under the boundedness of iterates assumption. In the future, we would like to incorporate function approximation architecture and apply our proposed algorithm on practical applications. \textcolor{black}{Moreover, as a future work, we would like to explore the theoretical sample complexity of our algorithm and compare the same with minimax Q-learning.}

\section{Acknowledgements}
Raghuram Bharadwaj was supported by a fellowship grant from the Centre for Networked Intelligence (a Cisco CSR initiative) of the Indian Institute of Science, Bangalore. Shalabh Bhatnagar was supported by the J.C.Bose Fellowship, a project from DST under the ICPS Program and the RBCCPS, IISc.


\bibliographystyle{IEEEtran}
\bibliography{references}
\end{document}